\documentclass[a4paper]{article}

\usepackage[english]{babel}
\usepackage[utf8x]{inputenc}
\usepackage[T1]{fontenc}
\usepackage{empheq}

\usepackage{latexsym}
\usepackage{amssymb}
\usepackage{amsthm}
\usepackage{amsmath}

\usepackage[a4paper,top=3cm,bottom=2cm,left=3cm,right=3cm,marginparwidth=1.75cm]{geometry}

\usepackage{amsmath,amssymb, graphics, setspace}
\usepackage{graphicx}
\usepackage[colorinlistoftodos]{todonotes}
\usepackage[colorlinks=true, allcolors=blue]{hyperref}
\usepackage{tikz}
\usetikzlibrary{shapes,calc,automata}
\usetikzlibrary{plothandlers}
\usepackage{graphicx}

\newcommand{\mathsym}[1]{{}}
\newcommand{\unicode}[1]{{}}
\newcommand{\args}{{\mathcal{A}}}
\newcommand{\attacks}{{\mathcal{R}}}
\newcommand{\values}{{\mathcal{V}}}

\newcommand{\attackers}[1]{\attacks^{-}(#1)}

\newtheorem{theorem}{Theorem}
\newtheorem{proposition}[theorem]{Proposition}

\newtheorem{example}{Example}

\title{A note on the uniqueness of models in social abstract argumentation}
\author{Leila Amgoud, Elise Bonzon, Marco Correia, Jorge Cruz, J\'er\^ome Delobelle, \\ S\'ebastien Konieczny, Jo\~ao Leite, Alexis Martin, Nicolas Maudet, Srdjan Vesic}

\begin{document}
\maketitle

\begin{abstract}
Social abstract argumentation is a principled way to assign values to conflicting (weighted) arguments. In this note we discuss the important property of the uniqueness of the model. 
\end{abstract}

\section{Introduction}
In the presence of several conflicting arguments, as is the case for instance in online debate platforms, it is difficult to appreciate which ones are ``stronger'' or more convincingly defended. This is even more challenging when those arguments can have weights (for instance, providing from votes of the users). 
To this date, social abstract argumentation \cite{LM11} remains one of the only proposal embracing both aspects of the structural interaction among arguments (attacks), and votes on these arguments. Furthermore, the model has been the subject of an in-depth algorithmic study and experimental evaluation \cite{CCL14}, showing that instances of significant size can be handled. 


Formally, a \emph{social abstract argumentation framework} is defined as a triple $\langle \args , \attacks, \values \rangle$, where $\args$ stands for a set of arguments, $\attacks \subseteq \args \times \args$ for a set of attacks among arguments, and $\values : \args \mapsto \mathbb{N} \times \mathbb{N}$ a function mapping each arguments to a tuple of \emph{pro} and \emph{con} votes ($v^+$ and $v^-$). We denote by $ \attackers{a}$ the set of attacker(s) of an argument $a$. 

Given a totally ordered set $L$ containing all possible valuations of an argument, with top ($\top$) and bottom ($\bot$) elements, 
finding a \emph{social model} consists in finding a mapping  $M_{S}: \args \rightarrow L$ of values to arguments satisfying certain constraints induced by the argumentation framework.  That is, for all $a  \in \args $ it must be the case that: 

\begin{empheq}[]{align}
\label{eq:social-model}
		M_{S}(a) = \tau(a) \curlywedge \neg \curlyvee\{M(a_{i}) : a_{i} \in \attackers{a}\}  
\end{empheq}
		where:
		$\tau : \args \rightarrow L$ is the social support; 
		$\curlywedge : L \times L \rightarrow L$ combines the initial score with the score of direct attackers; 
		$\curlyvee : L \times L \rightarrow L$ aggregates the score of direct attackers; 	
		and $\neg : L \rightarrow L$ restricts the value of the attacked argument. 
For the semantics to be \emph{well-behaved}, it is required that $\curlywedge$ is continuous, commutative, associative, monotonic w.r.t. both values and $\top$ is its identity element; that $\curlyvee$ is continuous, commutative, associative, monotonic w.r.t. both values and $\bot$ is its identity element; and that $\neg$ is antimonotonic, continuous, $\neg \bot = \top$, $\neg \top = \bot$ and $\neg \neg a = a$. $\tau$ is monotonic w.r.t. the first value and anti-monotonic w.r.t the second value. 
Most importantly, well-behaved semantics induce at least one social model \cite[Theorem 12]{LM11}.

A very important property that the designer may require is the \emph{uniqueness of the model}, that is, the fact that there exists only one social model satisfying Equation (\ref{eq:social-model}). 
One natural well-behaved social semantics satisfying this under some condition is the {\it simple product semantics}, defined as 
$\emph{SP}_{\epsilon} = \langle [0,1], \tau_{\epsilon},\curlywedge,\curlyvee, \neg \rangle$ where $\tau_{\epsilon} = \frac{v^+}{v^+ + v^- + \epsilon}$ (with $\epsilon > 0$, and $\tau_{\epsilon} = 0$ when there is no votes), $x_{1} \curlywedge x_{2} = x_{1} \times x_{2}$ (Product T-Norm),  $x_{1} \curlyvee x_{2} = x_{1} + x_{2} - x_{1} \times x_{2}$ (Probabilistic Sum T-CoNorm) and $\neg x_{1} = 1 - x_{1}$.

Indeed, when $|\attackers{a}| \times \tau(a) < 1$, for all arguments $a$ in the system, the uniqueness of the model can be proven \cite[Theorem 13]{LM11}. 
It is also conjectured \cite[Conjecture 14]{LM11} that this simple product models enjoy this property for social abstract argumentation frameworks \emph{in general}. In this note we show that while the conjecture holds when there are 3 arguments, it does not from 4 arguments onwards (Section \ref{sec:3arguments}). 
We briefly discuss the consequences of this finding in Section \ref{sec:discussion}.  
\section{Uniqueness of models only holds up to 3 arguments}
\label{sec:3arguments}

\begin{proposition}
In any social abstract argumentation frameworks $\langle \args , \attacks, \values \rangle$ with $|\args| \leq 3$, there exists a unique social model. 
\end{proposition}

\begin{proof}
We detail the case where $A= \{a_1,a_2,a_3\}$ and $R = A \times A$. The other cases are similar and can be derived along the same lines. 
We will prove that with \(0<a_1,a_2,a_3<1\), there is a unique solution solution in $(0,1)^3$ to the system:\\
\begin{empheq}[left=\empheqlbrace]{align}
x_1 = a_1 \left(1-x_2\right) \left(1-x_3\right)  \label{eq:1} \\ 
x_2 = a_2 \left(1-x_1\right) \left(1-x_3\right)  \label{eq:2} \\
x_3 = a_3 \left(1-x_1\right) \left(1-x_2\right) 
\end{empheq}




Let us consider the case where all \(a\){'}s are different, without loss of generality $a_3<a_2<a_1$ \\
Multiplying (\ref{eq:1}) and (\ref{eq:2}) by \(\left(1-x_1\right)\) and \(\left(1-x_2\right)\) respectively, we obtain:\\

\begin{empheq}[left=\empheqlbrace]{align*}
x_1\left(1-x_1\right) = a_1 \left(1-x_1\right)\left(1-x_2\right) \left(1-x_3\right) \\
x_2\left(1-x_2\right) = a_2 \left(1-x_1\right) \left(1-x_2\right)\left(1-x_3\right)
\end{empheq}

therefore
\[
\frac{x_1\left(1-x_1\right)}{a_1}=\frac{x_2\left(1-x_2\right)}{a_2}
\]
since \(x(1-x)=\)\(\frac{1}{4}-\left(\frac{1}{2}-x\right)^2\) we can rewrite the equation as:
\begin{empheq}[]{align}
\frac{a_2}{a_1}\left(\frac{1}{4}-\left(\frac{1}{2}-x_1\right){}^2\right)=\left(\frac{1}{4}-\left(\frac{1}{2}-x_2\right){}^2\right) \label{eq:3}
\end{empheq}
Since we assume that \(a_2<a_1\) it must be the case that:\\
\begin{empheq}{align*}
\frac{1}{4}-\left(\frac{1}{2}-x_1\right){}^2 & > \frac{1}{4}-\left(\frac{1}{2}-x_2\right){}^2 \\
\left(\frac{1}{2}-x_2\right){}^2 &>  \left(\frac{1}{2}-x_1\right){}^2\\
\sqrt{\left(\frac{1}{2}-x_2\right){}^2} &> \sqrt{\left(\frac{1}{2}-x_1\right){}^2} \\
\left|\frac{1}{2}-x_2\right| &>  \left|\frac{1}{2}-x_1\right| (\text{because } \sqrt{x^2}=|x|)
\end{empheq}

\noindent Let us know prove that \(x_2<\frac{1}{2}\). There are two cases for $x_1$: 
\begin{itemize}
\item Case \(x_1<\frac{1}{2}\). 
Suppose for contradiction that \(x_2\geq \frac{1}{2}\), therefore we must have 
 \(\left|\frac{1}{2}-x_2\right|>\left|\frac{1}{2}-x_1\right|\), thus \(x_2-\frac{1}{2}>\frac{1}{2}-x_1\), and finally \(x_2>1-x_1\). But this is impossible because \(x_2=a_2 \left(1-x_1\right) \left(1-x_3\right)\)\\
$\quad $Therefore, in this case \(x_2<\frac{1}{2}\) \\
\item Case \(x_1\geq \frac{1}{2}\). 
In that case we have \(-x_1\leq -\frac{1}{2}\), and thus 
\(1-x_1\leq \frac{1}{2}\). This means that \(x_2<\frac{1}{2}\) because \(x_2=a_2 \left(1-x_1\right) \left(1-x_3\right)\)\\
\end{itemize}
Therefore, in every case we have that \(x_2<\frac{1}{2}\)\\
Solving for \(x_2\) the equation (\ref{eq:3}) gives: 

\begin{empheq}{align*}
\left(\frac{1}{2}-x_2\right){}^2 & =\frac{1}{4}-\frac{a_2}{a_1}\left(\frac{1}{4}-\left(\frac{1}{2}-x_1\right){}^2\right)\\
x_2 & =\frac{1}{2}\pm \sqrt{\frac{1}{4}-\frac{a_2}{a_1}\left(\frac{1}{4}-\left(\frac{1}{2}-x_1\right){}^2\right)}
\end{empheq}
and since we know that \(x_2<\frac{1}{2}\) we have:
\begin{empheq}{align*}
x_2=\frac{1}{2}-\sqrt{\frac{1}{4}-\frac{a_2}{a_1}\left(\frac{1}{4}-\left(\frac{1}{2}-x_1\right){}^2\right)}
\end{empheq}
A similar reasoning can be used for $x_3$, 
therefore, the first equation of the initial system can be rewritten as:\\

\begin{empheq}{align*}
x_1 & =a_1 \left(1-\left(\frac{1}{2}-\sqrt{\frac{1}{4}-\frac{a_2}{a_1}\left(\frac{1}{4}-\left(\frac{1}{2}-x_1\right){}^2\right)}\right)\right)
\left(1-\left(\frac{1}{2}-\sqrt{\frac{1}{4}-\frac{a_3}{a_1}\left(\frac{1}{4}-\left(\frac{1}{2}-x_1\right){}^2\right)}\right)\right)\\
 & =a_1 \left(\frac{1}{2}+\sqrt{\frac{1}{4}-\frac{a_2}{a_1}\left(\frac{1}{4}-\left(\frac{1}{2}-x_1\right){}^2\right)}\right) \left(\frac{1}{2}+\sqrt{\frac{1}{4}-\frac{a_3}{a_1}\left(\frac{1}{4}-\left(\frac{1}{2}-x_1\right){}^2\right)}\right)
\end{empheq}
Let $f$ be the real variable function defined over \((0,1)\):
\begin{empheq}{align*}
f\left(x_1\right)=a_1 \left(\frac{1}{2}+\sqrt{\frac{1}{4}-\frac{a_2}{a_1}\left(\frac{1}{4}-\left(\frac{1}{2}-x_1\right){}^2\right)}\right)
\left(\frac{1}{2}+\sqrt{\frac{1}{4}-\frac{a_3}{a_1}\left(\frac{1}{4}-\left(\frac{1}{2}-x_1\right){}^2\right)}\right)
\end{empheq}
Any real value \(r_1\) with \(0<r_1<1\) that satisfies the equation \(r_1=f\left(r_1\right)\) determines one solution of the system:
\begin{empheq}{align*}
\left\langle x_1,x_2,x_3\right\rangle =\left\langle r_1,\frac{1}{2}-\sqrt{\frac{1}{4}-\frac{a_2}{a_1}\left(\frac{1}{4}-\left(\frac{1}{2}-r_1\right){}^2\right)},\frac{1}{2}-\sqrt{\frac{1}{4}-\frac{a_3}{a_1}\left(\frac{1}{4}-\left(\frac{1}{2}-r_1\right){}^2\right)}\right\rangle
\end{empheq}

Let us now prove that there exists one and only one value for \(0<r_1<1\) that satisfies equation \(r_1=f\left(r_1\right)\). 
\(f\) is a continuous function defined over \([0,1]\) (its roots always have positive values in this interval). 
The derivative $f'\left(x_1\right)$ of the function is:
\begin{empheq}{align*}
-\frac{a_2 \left(\frac{1}{2}+\sqrt{\frac{1}{4}-\frac{a_3}{a_1}\left(\frac{1}{4}-\left(\frac{1}{2}-x_1\right){}^2\right)}\right)
\left(\frac{1}{2}-x_1\right)}{\sqrt{\frac{1}{4}-\frac{a_2}{a_1}\left(\frac{1}{4}-\left(\frac{1}{2}-x_1\right){}^2\right)}}-\frac{a_3 \left(\frac{1}{2}+\sqrt{\frac{1}{4}-\frac{a_2}{a_1}\left(\frac{1}{4}-\left(\frac{1}{2}-x_1\right){}^2\right)}\right)
\left(\frac{1}{2}-x_1\right)}{\sqrt{\frac{1}{4}-\frac{a_3}{a_1}\left(\frac{1}{4}-\left(\frac{1}{2}-x_1\right){}^2\right)}}
\end{empheq}
\begin{itemize}
\item if $x_1<\frac{1}{2}$ both terms are negative and $f'\left(x_1\right)<0$ ($f$ is decreasing in $\left.\left[0,\frac{1}{2}\right.\right)$ 
\item if $x_1=\frac{1}{2}$ both terms are zero and $f'\left(x_1\right)=0$($f \left(\frac{1}{2}\right)$ is smaller value $f$ in $[0,1]$) 
\item if \(x_1>\frac{1}{2}\) both terms are positive and \(f'\left(x_1\right)>0\) (\(f\) is increasing in \(\left.\left(\frac{1}{2},1\right.\right]\))
\end{itemize}
Evaluating the function in its significant points, we get that: 

\begin{itemize}
\item $f(0)=a_1$
\item $f\left(\frac{1}{2}\right)=a_1 \left(\frac{1}{2}+\frac{1}{2}\sqrt{1-\frac{a_2}{a_1}}\right) \left(\frac{1}{2}+\frac{1}{2}\sqrt{1-\frac{a_3}{a_1}}\right)$
\item $f(1)=a_1$
\end{itemize}

The second derivative of the function is:\\

\begin{empheq}{align*}
f\text{''}\left(x_1\right) & =\frac{a_2 \left(\frac{1}{2}+\sqrt{\frac{1}{4}-\frac{a_3}{a_1}\left(\frac{1}{4}-\left(\frac{1}{2}-x_1\right){}^2\right)}\right)}{\sqrt{\frac{1}{4}-\frac{a_2}{a_1}\left(\frac{1}{4}-\left(\frac{1}{2}-x_1\right){}^2\right)}} 
- \frac{a_2{}^2
\left(\frac{1}{2}+\sqrt{\frac{1}{4}-\frac{a_3}{a_1}\left(\frac{1}{4}-\left(\frac{1}{2}-x_1\right){}^2\right)}\right) \left(\frac{1}{2}-x_1\right){}^2}{a_1
\left(\frac{1}{4}-\frac{a_2}{a_1}\left(\frac{1}{4}-\left(\frac{1}{2}-x_1\right){}^2\right)\right){}^{3/2}} \\
 & +\frac{a_3
\left(\frac{1}{2}+\sqrt{\frac{1}{4}-\frac{a_2}{a_1}\left(\frac{1}{4}-\left(\frac{1}{2}-x_1\right){}^2\right)}\right)}{\sqrt{\frac{1}{4}-\frac{a_3}{a_1}\left(\frac{1}{4}-\left(\frac{1}{2}-x_1\right){}^2\right)}} 
-\frac{a_3{}^2 \left(\frac{1}{2}+\sqrt{\frac{1}{4}-\frac{a_2}{a_1}\left(\frac{1}{4}-\left(\frac{1}{2}-x_1\right){}^2\right)}\right)
\left(\frac{1}{2}-x_1\right){}^2}{a_1 \left(\frac{1}{4}-\frac{a_3}{a_1}\left(\frac{1}{4}-\left(\frac{1}{2}-x_1\right){}^2\right)\right){}^{3/2}} \\
& + \frac{2
a_2 a_3 \left(\frac{1}{2} - x_1\right){}^2}{a_1 \sqrt{\frac{1}{4}-\frac{a_2}{a_1}\left(\frac{1}{4}-\left(\frac{1}{2}-x_1\right){}^2\right)} \sqrt{\frac{1}{4}-\frac{a_3}{a_1}\left(\frac{1}{4}-\left(\frac{1}{2}-x_1\right){}^2\right)}}
\end{empheq}

\begin{empheq}{align}
f\text{''}\left(x_1\right) & =\frac{a_2 \left(\frac{1}{2}+\sqrt{\frac{1}{4}-\frac{a_3}{a_1}\left(\frac{1}{4}-\left(\frac{1}{2}-x_1\right){}^2\right)}\right)}{\sqrt{\frac{1}{4}-\frac{a_2}{a_1}\left(\frac{1}{4}-\left(\frac{1}{2}-x_1\right){}^2\right)}}\left(1-\frac{a_2\text{
 }\left(\frac{1}{2}-x_1\right){}^2}{a_1 \left(\frac{1}{4}-\frac{a_2}{a_1}\left(\frac{1}{4}-\left(\frac{1}{2}-x_1\right){}^2\right)\right)}\right) \\
& + \frac{a_3
\left(\frac{1}{2}+\sqrt{\frac{1}{4}-\frac{a_2}{a_1}\left(\frac{1}{4}-\left(\frac{1}{2}-x_1\right){}^2\right)}\right)}{\sqrt{\frac{1}{4}-\frac{a_3}{a_1}\left(\frac{1}{4}-\left(\frac{1}{2}-x_1\right){}^2\right)}}
\left(1-\frac{a_3\text{
 }\left(\frac{1}{2}-x_1\right){}^2}{a_1 \left(\frac{1}{4}-\frac{a_3}{a_1}\left(\frac{1}{4}-\left(\frac{1}{2}-x_1\right){}^2\right)\right)}\right)\\
 & + \frac{2
a_2 a_3 \left(\frac{1}{2}-x_1\right){}^2}{a_1 \sqrt{\frac{1}{4}-\frac{a_2}{a_1}\left(\frac{1}{4}-\left(\frac{1}{2}-x_1\right){}^2\right)} \sqrt{\frac{1}{4}-\frac{a_3}{a_1}\left(\frac{1}{4}-\left(\frac{1}{2}-x_1\right){}^2\right)}} \label{eq:4}
\end{empheq}

Let us now inspect the terms of this expression. 
The last term (\ref{eq:4}) is clearly positive. 
The first two are also positive if \(\left(1-\frac{a\text{  }\left(\frac{1}{2}-x_1\right){}^2}{a_1 \left(\frac{1}{4}-\frac{a}{a_1}\left(\frac{1}{4}-\left(\frac{1}{2}-x_1\right){}^2\right)\right)}\right)>0\)
for any value of \(0<a<a_1\).\\
Let us consider its derivative. 
\begin{empheq}{align*}
\left(1-\frac{a\text{  }\left(\frac{1}{2}-x_1\right){}^2}{a_1 \left(\frac{1}{4}-\frac{a}{a_1}\left(\frac{1}{4}-\left(\frac{1}{2}-x_1\right){}^2\right)\right)}\right)'=\frac{2
a \left(\frac{1}{2}-x_1\right)}{a_1 \left(\frac{1}{4}-\frac{a \left(\frac{1}{4}-\left(\frac{1}{2}-x_1\right){}^2\right)}{a_1}\right)}\left(1-\frac{
a \left(\frac{1}{2}-x_1\right){}^2}{a_1 \left(\frac{1}{4}-\frac{a \left(\frac{1}{4}-\left(\frac{1}{2}-x_1\right){}^2\right)}{a_1}\right)}\right)
\end{empheq}

Whose only zero is when \(x_1=\frac{1}{2}\) (it can be checked that the second term of the product cannot be equal to 0, because $a<a_1$). 
For the interval \(\left.\left[0,\frac{1}{2}\right.\right)\) the derivative has values:\\
\begin{empheq}{align*}
\left(1-\frac{a\text{  }\left(\frac{1}{2}-0\right)^2}{a_1 \left(\frac{1}{4}-\frac{a}{a_1}\left(\frac{1}{4}-\left(\frac{1}{2}-0\right)^2\right)\right)}\right)'& =
\frac{a }{a_1 \left(\frac{1}{4}-0\right)}\left(1-\frac{ a \frac{1}{4}}{a_1 \left(\frac{1}{4}-0\right)}\right)=\frac{ 4a }{a_1 }\left(1-\frac{ a }{a_1}\right)>0
\end{empheq}

For the interval \(\left.\left(\frac{1}{2},1\right.\right]\) the derivative has values:
\begin{empheq}{align*}
\left(1-\frac{a\text{  }\left(\frac{1}{2}-1\right)^2}{a_1 \left(\frac{1}{4}-\frac{a}{a_1}\left(\frac{1}{4}-\left(\frac{1}{2}-1\right)^2\right)\right)}\right)' & =\frac{-a }{a_1 \left(\frac{1}{4}-0\right)}\left(1-\frac{ a \frac{1}{4}}{a_1 \left(\frac{1}{4}-0\right)}\right)=\frac{- 4a }{a_1 }\left(1-\frac{ a }{a_1}\right)<0
\end{empheq}
Therefore, the term \(\left(1-\frac{a\text{  }\left(\frac{1}{2}-x_1\right){}^2}{a_1 \left(\frac{1}{4}-\frac{a}{a_1}\left(\frac{1}{4}-\left(\frac{1}{2}-x_1\right){}^2\right)\right)}\right)\)
increases in \(\left.\left[0,\frac{1}{2}\right.\right)\) and decreases in \(\left.\left(\frac{1}{2},1\right.\right]\). Its least value in the interval is with \(x_1=0\) (with \(x_1=1\) the value is equal):\\
$\quad $\(\left(1-\frac{a\text{  }\left(\frac{1}{2}-x_1\right){}^2}{a_1 \left(\frac{1}{4}-\frac{a}{a_1}\left(\frac{1}{4}-\left(\frac{1}{2}-x_1\right){}^2\right)\right)}\right)>1-\frac{a\text{
 }\left(\frac{1}{2}-0\right)^2}{a_1 \left(\frac{1}{4}-\frac{a}{a_1}\left(\frac{1}{4}-\left(\frac{1}{2}-0\right)^2\right)\right)}=1-\frac{a \frac{1}{4}}{a_1
\left(\frac{1}{4}\right)}=1-\frac{a }{a_1 }>0\)\\
Therefore, every term of the second derivative are positive and therefore the second derivative is always positive in the interval.\\
\\
Summing up, in the interval \([0,1]\) we have:\\
\[
\begin{array}{c|ccccc}
  x_1 & [0] & (0,\frac{1}{2}) & \left[\frac{1}{2}\right] &
 \left(\frac{1}{2},1\right) & [1] \\
  \hline
f(x_1) & a_1 & \searrow & f(\frac{1}{2}) & \nearrow & a_1 \\  
f'(x_1) & - & - & 0 & + & + \\
f''(x_1) & + & + & + & + & + \\
\end{array}
\]

We will first prove that in \((0,1)\) there exists at least one solution to the equation \(x_1=f\left(x_1\right)\). { } \\
Note that any solution to the equation in \([0,1]\) corresponds to an intersection point between the straight line \(y=x\) and the curve
$y=f(x)$. \\
Both functions are continuous and are defined over the entire interval \([0,1]\). At one end of the interval, with \(x=0\),
the value of \(y=0\) that satisfies the equation of the straight line is inferior to the value of \(y=f(0)=a_1\) that satisfies the equation of the curve (the straight line is below the curve at the point \(x=0\)). At the other end of the interval, with \(x=1\), the value of \(y=1\) that satisfies the equation of the straight line is superior to the value of \(y=f(1)=a_1\) that satisfies the equation of the curve (the straight line is above the curve at the point \(x=1\)).
Since both the straight line and the curve are continuous, there has to be at least one point in which they cross, i.e.: 
$\quad $\(\exists _{r\in (0,1)}r=f(r)\)\\
\\
We shall now prove that in \((0,1)\) there exists at most one solution to the equation \(x_1=f\left(x_1\right)\).$\quad $\\
Note that the function \(f\) is strictly convex in \([0,1]\) since its second derivative is always positive in all points of the interval. 
Being convex in \([0,1]\) we know that:
\begin{itemize}
\item $\forall _{x_1\neq x_2\in [0,1]}\forall _{t\in (0,1)}f\left(\text{tx}_2+(1-t)x_1\right)<\text{tf}\left(x_2\right)+(1-t)f\left(x_1\right)$
\item $\forall _{x_1\neq x_2\in [0,1]}\forall _{t\in (0,1)}f\left(x_1+t\left(x_2-x_1\right)\right)<f\left(x_1\right)+t\left(f\left(x_2\right)-f\left(x_1\right)\right)$
\end{itemize}
Informally, the previous inequality establishes that the value of the function in any point between \(x_1\) and \(x_2\) is strictly smaller than the value of the point in the straight line conecting \(x_1\) to \(x_2\). \\
Let { }\(x_1=r_1\) be the smaller value of{ }\(x_1\) that satisfies the equation \(r_1=f\left(r_1\right)\) (we already established that there exists at least one). 
If \(r_1\) is the smaller value that satisfies the equation, then there are no additional solutions with values of \(x_1\in
\left(0,r_1\right)\). Let \(x_2=1\), through the inequality above we have:
\begin{empheq}{align*}
\forall_{t\in (0,1)}f\left(r_1+t\left(1-r_1\right)\right) & <f\left(r_1\right)+t\left(f(1)-f\left(r_1\right)\right) & \\
& < f\left(r_1\right)+t\left(a_1-f\left(r_1\right)\right) &  (\text{because } f(1)=a_1)\\
& <r_1+t\left(a_1-r_1\right) &  (\text{because } f\left(r_1\right)=r_1)\\
& < r_1+t\left(1-r_1\right) & (\text{because } a_1<1)
\end{empheq}

Through the previous inequality we know that no value of \(x_1\in \left(r_1,1\right)\) can satisfy equation \(x_1=f\left(x_1\right)\). 
Therefore, \(r_1\) is the only solution in \((0,1)\), so there exists one and only one solution to equation \(x_1=f\left(x_1\right)\) in this interval. 
Consequently, there is also one and only one solution of the original system in \((0,1)^3\). 
\end{proof}


The question is thus whether the same property holds for a larger number of arguments, as conjectured in \cite{LM11}. It turns out that, from 4 arguments, several social models may exist. 


\begin{example}[Non-uniqueness of models]
The example involves four arguments involved in pairwise reciprocal attacks, as illustrated (see Figure \ref{figNonUniquenessFourArgs}). 
\begin{figure}[h]
\begin{center}
\begin{tikzpicture}[scale=1]
\tikzstyle{arg-out}=[circle,draw, line width=1pt]

\node (argA) at (0,2.5) [arg-out,label=above left:\footnotesize{$(1,0)$}, label=below:\footnotesize{}] {$a$};
\node (argB) at (1.5,2.5) [arg-out,label=above right:\footnotesize{$(1,0)$}, label=below:\footnotesize{}] {$b$};
\node (argD) at (0,1) [arg-out,label=below left:\footnotesize{$(1,0)$}, label=below:\footnotesize{}] {$d$};
\node (argC) at (1.5,1) [arg-out,label=below right:\footnotesize{$(1,0)$}, label=below:\footnotesize{}] {$c$};

\draw [->] (argA) to [bend left=20] (argB);
\draw [->] (argB) to [bend left=20] (argA);
\draw [->] (argB) to [bend left=20] (argC);
\draw [->] (argC) to [bend left=20] (argB);
\draw [->] (argC) to [bend left=20] (argD);
\draw [->] (argD) to [bend left=20] (argC);
\draw [->] (argD) to [bend left=20] (argA);
\draw [->] (argA) to [bend left=20] (argD);

 
 \end{tikzpicture}
\end{center}
\caption{\label{figNonUniquenessFourArgs} Example of AF with multiple valid models}
\end{figure}
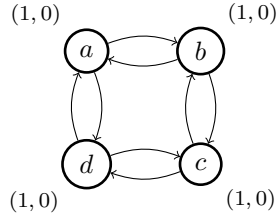
~\\
To compute the score of each argument, we will use the {\it simple product semantics} introduced as a well-behaved SAF semantics in \cite{LM11}. 
Recall that the vote aggregation function $\tau_{\epsilon}$ (with $\epsilon > 0$) corresponds to the proportionality of positive votes by the total number of votes on an argument.
Thus, by taking $\epsilon = 0.1$, the social support of all the arguments is: $\tau_{\epsilon}(a) = \tau_{\epsilon}(b) = \tau_{\epsilon}(c) = \tau_{\epsilon}(d) = \frac{1}{1 + 0 + \epsilon} = \frac{1}{1.1} \approx 0.909$.\\
Let us now write the equation system (with one equation for each argument) from the AF illustrated in Figure \ref{figNonUniquenessFourArgs}.
$$
\begin{cases} 
	M(a) = M(c) = \frac{1}{1.1} \times ( 1 - (M(b) + M(d) - M(b) \times M(d))) \\ 
	M(b) = M(d) = \frac{1}{1.1} \times ( 1 - (M(a) + M(c) - M(a) \times M(c))) 
\end{cases}
$$
$$
\begin{cases} 
	M(a) = M(c) = \frac{1}{1.1} \times ( 1 - (2M(b) - M(b)^{2})) = \frac{1}{1.1} \times ( 1 - (2M(d) - M(d)^{2}))  \\ 
	M(b) = M(d) = \frac{1}{1.1} \times ( 1 - (2M(a) - M(a)^{2})) = \frac{1}{1.1} \times ( 1 - (2M(c) - M(c)^{2})) 
\end{cases}
$$

\noindent It can be checked that this equation system has three distinct valid models which lead to different rankings: 
\begin{table}[h]
$$
\begin{array}{c|cccc|c}
&M(a)  &  M(b)   &    M(c)     &   M(d) & ranking  \\
\hline 
\text{model} \ 1  &0.36573 &  0.36573 &  0.36573 &  0.36573 & a \simeq b \simeq c \simeq d  \\
\hline
\text{model} \ 2 &0.01125 & 0.88875 & 0.01125 & 0.88875 & b \simeq d \succ a \simeq c \\
\hline
\text{model}\ 3 &0.88875 & 0.01125 &  0.88875 & 0.01125 & a \simeq c \succ b \simeq d \\
\hline
\end{array} 
$$
\caption{\label{tabDifferentModel} The three distinct valid models from the AF illustrated in Figure \ref{figNonUniquenessFourArgs}}
\end{table}
\end{example}

\section{Discussion}
\label{sec:discussion}

We conclude by briefly discussing the consequences of this result. 
One possible way to ensure the uniqueness of the model would be to normalize the values, for instance by dividing the initial social support $\tau(a)$ of all arguments by the overall number of arguments in the system, thus making sure that $|\attackers{a}| \times \tau(a) < 1$, compute the model (and multiply again). When one is only interested in the rankings of arguments (\emph{i.e.} not in their absolute values), this is a sufficient repair strategy. 

However, this normalization implies that some desirable properties (which were previously satisfied by social abstract argumentation frameworks, at least in the restricted setting of a unique positive vote, see \cite{BDKM16}) would now be falsified. 
One of these properties --satisfied by many of the existing semantics-- is the property of ``Ordinal Independence'' \cite{AB16}. It says that the ranking between two arguments should be independent of arguments that are not connected to either of them. The following example shows that it cannot be guaranteed any longer. 

\begin{example}[Ordinal independence is not satisfied]
We first consider the situation where there are 6 arguments. We shall focus here on the value of some argument in the cycle $(a,b,c)$, say $a$, as compared to the value of argument $f$. 
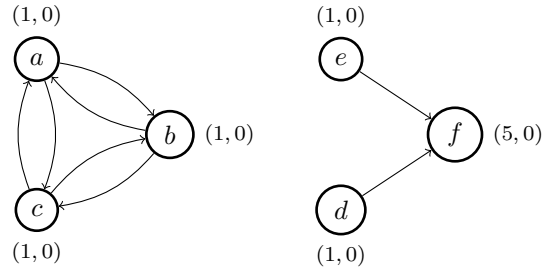
\begin{center}
\begin{tikzpicture}[scale=1]
\tikzstyle{arg-out}=[circle,draw, line width=1pt]
\tikzstyle{arg-in}=[circle,draw, fill=black!50, line width=1pt]

\node (argA) at (0,3) [arg-out,label=above:\footnotesize{$(1,0)$}, label=below:\footnotesize{}] {$a$};
\node (argB) at (1.75,2) [arg-out,label=right:\footnotesize{$(1,0)$}, label=below:\footnotesize{}] {$b$};
\node (argC) at (0,1) [arg-out,label=below:\footnotesize{$(1,0)$}, label=below:\footnotesize{}] {$c$};
\node (argD) at (4,1) [arg-out,label=below:\footnotesize{$(1,0)$}, label=below:\footnotesize{}] {$d$};
\node (argE) at (4,3) [arg-out,label=above:\footnotesize{$(1,0)$}, label=below:\footnotesize{}] {$e$};
\node (argF) at (5.5,2) [arg-out,label=right:\footnotesize{$(5,0)$}, label=below:\footnotesize{}] {$f$};

\draw [->] (argC) to [bend left=20] (argA);
\draw [->] (argA) to [bend left=20] (argC);
\draw [->] (argB) to [bend left=20] (argA);
\draw [->] (argA) to [bend left=20] (argB);
\draw [->] (argB) to [bend left=20] (argC);
\draw [->] (argC) to [bend left=20] (argB);
\draw [->] (argE) to (argF);
\draw [->] (argD) to (argF);
 
 \end{tikzpicture}
\end{center}
Computing the values (again in using the simple product semantics with $\epsilon =0.1$) gives $M(a)=0.7074$ and $M(f)=0.7058$, hence $a \succ f$. 
But now suppose there are 1000 more --completely unrelated-- arguments. By using the same method, we get that $M(a)=0.9074$, while $M(f)=0.97$, hence $f \succ a$. The addition of these arguments have modified the relative ranking of $a$ and $f$. 

\end{example}

\bibliographystyle{alpha}


\end{document}